\documentclass[11pt]{article}

\usepackage{tikz}

\usepackage[utf8]{inputenc}
\usepackage{enumerate}
\usepackage{chngpage}
\usepackage{amsfonts, amsmath, amssymb, mathrsfs}
\usepackage[amsmath,hyperref,amsthm]{ntheorem}

\def\dot#1#2{\langle #1,#2 \rangle}

\def\calB{\mathcal{B}}
\def\E{\mathbf{E}}

\newcommand{\mlower}{M_{L}}

\def\calD{\mathcal{D}}
\newcommand{\bmu}{\text{\boldmath$\mu$}}

\newcommand{\hu}{\breve{u}}
\newcommand{\hv}{\breve{v}}

\newcommand{\hup}{\breve{u}^{\perp}}
\newcommand{\bbR}{{\mathbf{R}}}
\newcommand{\bbE}{{\mathbf{E}}}
\newcommand{\Mu}{{\mathcal{M}}}
\newcommand{\imeans}{2-means-iterate}
\newcommand{\samp}{\mathcal{S}}
\newcommand{\hS}{\hat{S}}

\newtheorem{whatever}{Whatever}
\newtheorem{assumption}[whatever]{Assumption}
\newtheorem{lemma}[whatever]{Lemma}
\newtheorem{theorem}[whatever]{Theorem}
\newtheorem{corollary}[whatever]{Corollary}

\newcommand{\qed}{\hfill $\Box$ \hfill}
\newcommand{\m}{m}

\newcommand{\mmu}{\mathbf{\mu}}
\newcommand{\kl}{\mathbf{KL}}

\newcommand{\calN}{\mathcal{N}}
\newcommand{\calF}{\mathcal{F}}
\newcommand{\kc}[1]{\noindent{\textcolor{red}{\{{\bf KC:} \em #1\}}}}

\title{Learning Mixtures of Gaussians Using the $k$-Means Algorithm}
\author{Kamalika Chaudhuri \\
 CSE Department, UC San Diego \\
{\tt{kamalika@soe.ucsd.edu}}
\and
Sanjoy Dasgupta \\
CSE Department, UC San Diego\\
{\tt{dasgupta@cs.ucsd.edu}}
\and Andrea Vattani \\
CSE Department, UC San Diego\\
{\tt{avattani@cs.ucsd.edu}}
}

\begin{document}

\maketitle
\begin{abstract}
One of the most popular algorithms for clustering in Euclidean space is the $k$-means algorithm; $k$-means is difficult to analyze mathematically, and few theoretical guarantees are known about it, particularly when the data is {\em well-clustered}. In this paper, we attempt to fill this gap in the literature by analyzing the behavior of $k$-means on well-clustered data. In particular, we study the case when each cluster is distributed as a different Gaussian -- or, in other words, when the input comes from a mixture of Gaussians. 

We analyze three aspects of the $k$-means algorithm under this assumption. First, we show that when the input comes from a mixture of two spherical Gaussians, a variant of the $2$-means algorithm successfully isolates the subspace containing the means of the mixture components. Second, we show an exact expression for the convergence of our variant of the $2$-means algorithm, when the input is a very large number of samples from a mixture of spherical Gaussians. Our analysis does not require any
lower bound on the separation between the mixture components. 

Finally, we study the sample requirement of $k$-means; for a mixture of $2$ spherical Gaussians, we show an upper bound on the number of samples required by a variant of $2$-means to get close to the true solution. The sample requirement grows with increasing dimensionality of the data, and decreasing separation between the means of the Gaussians. To match our upper bound, we show an information-theoretic lower bound on any algorithm that learns mixtures of two spherical Gaussians; our lower bound indicates that in the case when the overlap between the probability masses of the two distributions is small, the sample requirement of $k$-means is {\em near-optimal}.
\end{abstract}

\newpage
\section{Introduction}

One of the most popular algorithms for clustering in Euclidean space is the $k$-means algorithm~\cite{L82, F65, M67}; this is a simple, local-search algorithm that iteratively refines a partition of the input points until convergence. Like many local-search algorithms, $k$-means is notoriously difficult to analyze, and few theoretical guarantees are known about it.

There has been three lines of work on the $k$-means algorithm. A first line of questioning addresses the quality of the solution produced by $k$-means, in comparison to the globally optimal solution. While it has been well-known that for general inputs the quality of this solution can be arbitrarily bad, the conditions under which $k$-means yields a globally optimal solution on {\em well-clustered} data are not well-understood. A second line of work~\cite{AV06, V09} examines the number of iterations required by $k$-means to converge. ~\cite{V09} shows that there exists a set of $n$ points on the plane, such that $k$-means takes as many as $\Omega(2^n)$ iterations to converge on these points. A smoothed analysis upper bound of $poly(n)$ iterations has been established by~\cite{AMR09}, but this bound is still much higher than what is observed in practice, where the number of iterations are frequently sublinear in $n$. Moreover, the smoothed analysis bound applies to small perturbations of arbitrary inputs, and the question of whether one can get faster convergence on well-clustered inputs, is still unresolved. A third question, considered in the statistics literature, is the statistical efficiency of $k$-means. Suppose the input is drawn from some simple distribution, for which $k$-means is statistically consistent; then, how many samples is required for $k$-means to converge? Are there other consistent procedures with a better sample requirement?

In this paper, we study all three aspects of $k$-means, by studying the behavior of $k$-means on Gaussian clusters.
Such data is frequently modelled as a mixture of Gaussians; a mixture is a collection of Gaussians $\calD = \{ D_1, \ldots, D_k \}$ and
weights $w_1, \ldots, w_k$, such that $\sum_i w_i = 1$. To sample from the
mixture, we first pick $i$ with probability $w_i$ and then draw a random
sample from $D_i$. Clustering such data then reduces to the problem of {\em learning a mixture}; here, we are given only
the ability to sample from a mixture, and our goal is to learn the
parameters of each Gaussian $D_i$, as well as
determine which Gaussian each sample came from. 

Our results are as follows. First, we show that when the input comes from a mixture of two spherical Gaussians, a variant of the $2$-means algorithm successfully isolates the subspace containing the means of the Gaussians. Second, we show an exact expression for the convergence of a variant of the $2$-means algorithm, when the input is a large number of samples from a mixture of two spherical Gaussians. Our analysis shows that the convergence-rate is logarithmic in the dimension, and decreases with increasing separation between the mixture components. Finally, we address the sample requirement of $k$-means; for a mixture of $2$ spherical Gaussians, we show an upper bound on the number of samples required by a variant of $2$-means to get close to the true solution. The sample requirement grows with increasing dimensionality of the data, and decreasing separation between the means of the distributions. To match our upper bound, we show an information-theoretic lower bound on any algorithm that learns mixtures of two spherical Gaussians; our lower bound indicates that in the case when the overlap between the probability masses of the two distributions is small, the sample requirement of $2$-means is {\em near-optimal}.

Additionally, we make some partial progress towards analyzing $k$-means in the
more general case -- we show that if our variant of $2$-means is run on a
mixture of $k$ spherical Gaussians, then, it converges to a vector in the
subspace containing the means of $D_i$.

The key insight in our analysis is a novel potential function $\theta_t$,
which is the minimum angle between the subspace of the means of $D_i$, and the normal to the hyperplane separator in $2$-means. We show that this angle decreases with iterations of our variant of $2$-means, and we can characterize convergence rates and sample requirements, by characterizing the rate of decrease of the potential.

\medskip\noindent{\textbf{Our Results.}} More specifically, our results are as follows. We perform a probabilistic
analysis of a variant of $2$-means; our variant is essentially a
symmetrized version of $2$-means, and it reduces to $2$-means when we have
a very large number of samples from a mixture of two identical spherical
Gaussians with equal weights. In the $2$-means algorithm, the separator
between the two clusters is always a hyperplane, and we use the angle $\theta_t$ 
between the normal to this hyperplane and the mean of a mixture component 
in round $t$, as a measure of the potential in each round. Note that 
when $\theta_t = 0$, we have arrived at the correct solution.

First, in Section~\ref{sec:k2infsamples}, we consider the case when we have
at our disposal a very large number of samples from a mixture of
$N(\mu^1,(\sigma^1)^2
I_d)$ and $N(\mu^2,(\sigma^2)^2 I_d)$ with mixing weights $\rho^1, \rho^2$
respectively.  We show an exact
relationship between $\theta_t$ and $\theta_{t+1}$, for any
value of $\mu^j$, $\sigma^j$, $\rho^j$ and $t$. Using this relationship, we
can approximate the rate of convergence of $2$-means, for different values of the separation, as well as different initialization procedures. Our guarantees illustrate that
the progress of $k$-means is very fast -- namely, the square of the cosine
of $\theta_t$ grows by at least a constant factor (for high separation)
each round, when one is far from the actual solution, and slow when the actual solution is very close. 

Next, in Section~\ref{sec:k2finsamples}, we characterize the sample
requirement for our variant of $2$-means to succeed, when the input is a
mixture of two spherical Gaussians. For the
case of two identical spherical Gaussians with equal mixing weight, our
results imply that when the separation $\mu < 1$, and when
$\tilde{\Omega}(\frac{d}{\mu^4})$ samples are used in each round, the
$2$-means algorithm makes progress at roughly the same rate as in
Section~\ref{sec:k2infsamples}. This agrees with the
$\Omega(\frac{1}{\mu^4})$ sample complexity lower bound~\cite{Lbook} for
learning a mixture of Gaussians on the line, as well as with experimental
results of~\cite{SSR06}. 
When $\mu > 1$, our variant of $2$-means makes progress in each round, when
the number of samples is at least $\tilde{\Omega}(\frac{d}{\mu^2})$.

Then, in Section~\ref{sec:lowerbounds}, we provide an information-theoretic
lower bound on the sample requirement of any algorithm for learning a
mixture of two spherical Gaussians with standard deviation $1$ and equal
weight. We show that when the separation $\mu > 1$, any algorithm requires
$\Omega(\frac{d}{\mu^2})$ samples to converge to a vector within angle
$\theta = \cos^{-1}(c)$ of the true solution, where $c$ is a constant. This
indicates that $k$-means has near-optimal sample requirement when $\mu > 1$.

Finally, in Section~\ref{sec:genkinf}, we examine the performance of
$2$-means when the input comes from a mixture of $k$ spherical Gaussians.
We show that, in this case, the normal to the hyperplane separating the two
clusters converges to a vector in the subspace containing the means of the
mixture components. Again, we characterize exactly the rate of convergence,
which looks very similar to the bounds in Section~\ref{sec:k2infsamples}.
\medskip\noindent{\textbf{Related Work.}}
The convergence-time of the $k$-means algorithm
has been analyzed in the worst-case~\cite{AV06, V09}, and the smoothed
analysis settings~\cite{MR09, AMR09}; ~\cite{V09} shows that the
convergence-time of $k$-means may be $\Omega(2^n)$ even in the
plane.~\cite{AMR09} establishes a $O(n^{30})$ smoothed complexity bound.
~\cite{ORSS06} analyzes the performance of $k$-means when the
data obeys a clusterability condition; however, their clusterability
condition is very different, and moreover, they examine conditions under
which constant-factor approximations can be found.
In statistics literature, the $k$-means algorithm has been shown to be
consistent~\cite{M67}.~\cite{P81} shows that minimizing the $k$-means
objective function (namely, the sum of the squares of the distances between
each point and the center it is assigned to), is consistent, given
sufficiently many samples. As optimizing the $k$-means objective is
NP-Hard, one cannot hope to always get an exact solution. None of these two works quantify either the convergence rate or the exact sample requirement of $k$-means.

There has been two lines of previous work on theoretical analysis of the EM
algorithm~\cite{DLR77}, which is closely related to $k$-means. Essentially,
for learning mixtures of identical Gaussians, the only difference between
EM and $k$-means is that EM uses {\em partial assignments} or {\em soft
clusterings}, whereas $k$-means does not. First, ~\cite{RW84, XJ96} views learning mixtures as an optimization problem, and EM as an optimization procedure over the likelihood surface. They analyze the structure of the likelihood surface around the optimum to conclude that EM has first-order convergence. An optimization procedure on a parameter
$m$ is said to have first-order convergence, if, 
\[ ||m_{t+1} - m^*|| \leq R \cdot ||m_t - m^*|| \]
where $m_t$ is the estimate of $m$ at time step $t$ using $n$
samples, $m^*$ is the maximum likelihood estimator for $m$ using
$n$ samples, and $R$ is some fixed constant between $0$ and $1$. In contrast, our analysis also applies when one is far from the optimum.

The second line of work is a probabilistic analysis of EM due to~\cite{DS00}; they show a two-round variant of EM
which converges to the correct partitioning of the samples, when the input
is generated by a mixture of $k$ well-separated, spherical Gaussians.  For
their analysis to work, 
they require the mixture components to be separated such that two samples from the same Gaussian are a little closer in space than two
samples from different Gaussians. In contrast, our analysis applies when the separation is much smaller. 

The sample requirement of learning mixtures has been previously studied in
the literature, but not in the context of $k$-means. ~\cite{CHRZ07, C07} provides an algorithm that learns a
mixture of two binary product distributions with uniform weights, when the
separation $\mu$ between the mixture components is at least a constant, so
long as $\tilde{\Omega}(\frac{d}{\mu^4})$ samples are available. (Notice
that for such distributions, the directional standard deviation
is at most $1$.) Their algorithm is similar to $k$-means in some respects,
but different in that they use different sets of coordinates in each round,
and this is 
very crucial in their analysis. Additionally,~\cite{BCFZ07} show a spectral algorithm which learns a mixture of $k$ binary product
distributions, when the distributions have small overlap in probability
mass, and the sample size is at least $\tilde{\Omega}(d/\mu^2)$. \cite{Lbook} shows that at least $\tilde{\Omega}(\frac{1}{\mu^4})$ samples are required to learn a mixture of two Gaussians in one dimension.

We note that although our lower bound of $\Omega(d/\mu^2)$ for $\mu >1$ seems to contradict the upper bound of~\cite{CHRZ07, C07}, this is not actually the case. Our lower bound characterizes the number of samples required to find a vector at an angle $\theta = \cos^{-1}(1/10)$ with the vector joining the means. However, in order to classify a constant fraction of the points correctly, we only need to find a vector at an angle $\theta' = \cos^{-1}(1/\mu)$ with the vector joining the means. Since the goal of~\cite{CHRZ07} is to simply
classify a constant fraction of the samples, their upper bound is less than $O(d/\mu^2)$.

In addition to theoretical analysis, there has been very interesting
experimental work due to~\cite{SSR06}, which studies the sample requirement
for EM on a mixture of $k$ spherical Gaussians. They conjecture that the
problem of learning mixtures has three phases, depending on the number of
samples : with less than about $\frac{d}{\mu^4}$ samples, learning mixtures
is information-theoretically hard; with more than about $\frac{d}{\mu^2}$
samples, it is computationally easy, and in between, computationally hard,
but easy in an information-theoretic sense. Finally, there has been a line of work which provides algorithms (different from EM or $k$-means) that are guaranteed to learn mixtures of Gaussians under certain separation conditions -- see, for
example,~\cite{D99, VW02, AK01, AM05, KSV05, CR08, BV08}. For mixtures of
two Gaussians, our result is comparable to the best results for spherical
Gaussians~\cite{VW02} in terms of separation requirement, and we have a smaller sample requirement.

\section{The Setting}
\label{sec:setting}

The $k$-means algorithm iteratively refines a partitioning of the
input data. At each iteration, $k$ points are maintained as {\em centers};
each input is assigned to its closest center. The center of each cluster is
then recomputed as the empirical mean of the points assigned to the cluster. 
This procedure is continued until convergence.

Our variant of $k$-means is described below. There are two main differences between the actual $2$-means algorithm, and our
variant. First, we use a separate set of samples in each iteration.
Secondly, we always fix the cluster boundary to be a
hyperplane through the origin. When the input is a very large number of samples from a mixture of two identical Gaussians
with equal mixing weights, and with center of mass at the origin, 
this is exactly $2$-means initialized with symmetric centers (with respect to the origin).
We analyze this symmetrized version of $2$-means even when the mixing weights and the variances of the Gaussians
in the mixture are not equal.

The input to our algorithm is a set of samples $\samp$, a number of iterations $N$, and a starting vector $\hu_0$, and the output is a vector $u_N$ obtained after $N$ iterations of the $2$-means algorithm. 

\medskip \noindent \textbf{\imeans($\samp$, $N$, $u_0$)}
\begin{enumerate}

\item Partition $\samp$ randomly into sets of equal size $\samp_1, \ldots, \samp_N$.
\item For iteration $t = 0, \ldots, N-1$, compute:
\begin{eqnarray*}
C_{t+1} & = & \{ x \in \samp_{t+1} | \dot{x}{u_{t}} > 0 \} \\
\bar{C}_{t+1} & = & \{ x \in \samp_{t+1} | \dot{x}{u_t} < 0 \}
\end{eqnarray*}
Compute: $u_{t+1}$ as the empirical average of $C_{t+1}$.
\end{enumerate}

\noindent {\textbf{Notation.}} In Sections~\ref{sec:k2infsamples}
and~\ref{sec:k2finsamples}, we analyze Algorithm \imeans, when the input is
generated by a mixture $\calD = \{D_1, D_2\}$ of two Gaussians. We let $D_1
= N(\mu^1, (\sigma^1)^2 I_d)$, $D_2 = N(\mu^2, (\sigma^2)^2 I_d)$, with
mixing weights $\rho^1$ and $\rho^2$. We also assume without loss of
generality that for all $j$, $\sigma^j \geq 1$. As the center of mass of the mixture lies at the origin, $\rho^1 \mu^1 + \rho^2 \mu^2 = 0$. 
In Section~\ref{sec:genkinf}, we study a somewhat more general case.

We define $b$ as the unit vector along $\mu^1$, i.e. $b = \frac{\mu^1}{||\mu^1||}$.Henceforth, for any vector $v$, we use the notation $\hv$ to denote the unit vector along $v$, i.e. $\hv = \frac{v}{||v||}$. 
Therefore, $\hu_t$ is the unit vector along $u_t$. 
We assume without loss of generality that $\mu^1$ lies in the cluster $C_{t+1}$. 
In addition, for each $t$, we define $\theta_t$ as the angle between $\mu^1$
and $u_t$.We use the cosine of $\theta_t$ as a measure of progress of
the algorithm at round $t$, and our goal is to show that this quantity increases
as $t$ increases. Observe that $0 \leq \cos(\theta_t) \leq 1$, and
$\cos(\theta_t) = 1$ when $u_t$ and $\mu^1$ are aligned along the same
direction.  For each $t$, we define $\tau_t^j =\dot{\mu^j}{\hu_t} = \dot{\mu^j}{b}\cos(\theta_t)$. 
Moreover, from our notation, $\cos(\theta_t) = \frac{\tau_t^1}{||\mu^1||}$.
In addition, we define $\rho_{\min} = \min_j \rho^j$, $\mu_{\min} = \min_j
||\mu^j||$, and $\sigma_{\max} = \max_j \sigma^j$. For the special case of two identical spherical Gaussians with equal weights, we use $\bmu = ||\mu^1|| = ||\mu^2||$.
Finally, for $a \leq b$, we use the notation $\Phi(a,b)$ to denote the
probability that a standard normal variable takes values between $a$ and
$b$.

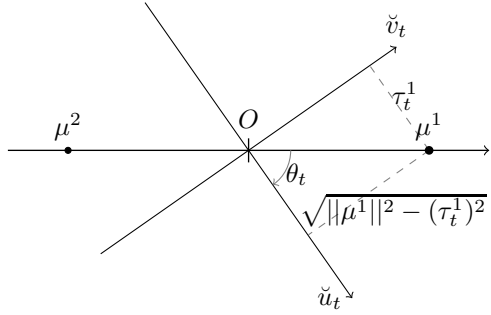
\begin{figure}[t]
\begin{center}

\begin{tikzpicture}[scale=0.8]
    \newcommand{\scalefont}[1]{{\small #1}}

    \path(0,0) ++(35:3)coordinate(v);
    \path(0,0) ++(-55:3)coordinate(u);
    \path(3,0) ++(-145:2.45) coordinate (u_intersect);
    \path(3,0) ++(125:1.7) coordinate (v_intersect);
    
    \draw[->,line width=0.5pt] (-4,0)--(4,0);
    \draw(0,0.2) node[above]{\scalefont{$O$}};
    \draw[line width=0.5pt] (0,-0.2)--(0,0.2);

    \draw(u) node[left]{\scalefont{$\breve{u}_t$}};
    \draw(v) node[above]{\scalefont{$\breve{v}_t$}};

    \draw[rotate=35,->](0,3)--(0,-3);
    \draw[rotate=35,->](-3,0)--(3,0);
    \draw[gray,->](0.7,0) arc (0:-55:0.7);
    \draw(0.8,-0.4)node{\scalefont{$\theta_t$}};

    \draw[gray, dashed](3,0)--(u_intersect);
    \draw(2.6,0.9) node{\scalefont{$\tau_t^1$}};
    \draw[gray, dashed](3,0)--(v_intersect);
    \draw(2.4,-1) node{\scalefont{$\sqrt{||\mu^1||^2-(\tau_t^1)^2}$}};

    \filldraw[line width=1pt](3,0)circle(.05);
    \draw(-3,0) node[above]{\scalefont{$\mu^2$}};
    \filldraw(-3,0)circle(.05);
    \draw(3,0) node[above]{\scalefont{$\mu^1$}};

\end{tikzpicture}

{\caption{\small Here we are depicting the plane defined by the vectors $\mu^1$ and $\hu_t$. 
The vector $\hv_t$ is simply the unit vector along $\mu^1 - \dot{\mu^1}{\hu_t}\hu_t$. 
Therefore, we have $\tau_t^1=||\mu^1||\cos(\theta_t)$ and
$\sqrt{||\mu^1||^2-(\tau_t^1)^2}=||\mu^1|| \sin(\theta_t)$.}}
\end{center}
\end{figure}

\iffalse
We would like to compute $u_{t+1}$, the center of $C_{t+1}$. We define:
\begin{eqnarray*}
w^1_{t+1} & = & \Pr[x \sim D_1 | x \in C_{t+1}] \\
w^2_{t+1} & = & \Pr[x \sim D_2 | x \in C_{t+1}] \\
c^1_{t+1} & = & \E[x | x \sim D_1, x \in C_{t+1}] \\
c^2_{t+1} & = & \E[x | x \sim D_2, x \in C_{t+1}]
\end{eqnarray*}

We observe that $c_{t+1}$, the expected center of cluster $C_{t+1}$ can be written
as:
\[ c_{t+1} =w^1_{t+1} c^1_{t+1} +w^2_{t+1} c^2_{t+1} \]

In addition, for each $t$, we define $\theta_t$ as the angle between $\mu$
and the expected center of the cluster $C_t$ containing $\mu$. Obviously,
closer $\theta_t$ is to $0$, the closer we are to the ideal solution of the
problem. We use the cosine of $\theta_t$ as a measure of the progress of
$k$-means at Step $t$, and our goal is to show that this quantity increases
as $t$ increases. Observe that $0 \leq \cos(\theta_t) \leq 1$, and
$\cos(\theta_t) = 1$ when $c_{t}$ and $\mu$ are aligned along the same
direction. We observe that, from our notation:
\[ \cos (\theta_t) = \frac{\tau_t}{\mu} \]

Finally, for $a \leq b$, we use the notation $P(a,b)$ to denote the
probability that a standard normal variable takes values between $a$ and
$b$.
\fi
\section{Exact Estimation}\label{sec:k2infsamples}

In this section, we examine the performance of Algorithm \imeans\ when one can estimate the vectors $u_{t}$ exactly -- that is, when a
very large number of samples from the mixture is available. 
Our main result of this section is Lemma \ref{lem:expr1}, which exactly characterizes the behavior of \imeans\ at a specific iteration $t$.

For any $t$, we define the quantities $\xi_t$ and $\m_t$ as follows:
\begin{eqnarray*} 
\xi_t =  \sum_j \rho^j \sigma^j\frac{e^{-(\tau_t^j)^2/2(\sigma^j)^2}}{\sqrt{2 \pi}}, & 
\m_t  =  \sum_j \rho^j \dot{\mu^j}{b} \cdot \Phi(-\frac{\tau_t^j}{\sigma^j},\infty)
\end{eqnarray*}

Now, our main lemma can be stated as follows.

\begin{lemma}\label{lem:expr1}
\[\cos^2(\theta_{t+1}) = \cos^2(\theta_t)  \left(1 + \tan^2(\theta_t) \frac{ 2\cos(\theta_t)\xi_t\m_t + \m_t^2 }{ \xi_t^2 + 2\cos(\theta_t)\xi_t\m_t + \m_t^2 } \right) \]
\end{lemma}

The proof is in the Appendix. Using Lemma~\ref{lem:expr1}, we can characterize the convergence rates and
times of \imeans\ for different values of $\mu^j$, $\rho^j$ and $\sigma^j$, as well as  different initializations of $u_0$. 

The convergence rates can be characterized in terms of two natural parameters of the problem, $M = \sum_j \frac{\rho^j ||\mu^j||^2}{\sigma^j}$, 
 which measures how much the distributions are separated, and  $V = \sum_j \rho^j \sigma^j$, which measures the average standard deviations of the distributions. We observe that as $\sigma^j \geq 1$, for all $j$, $V \geq 1$ always. 
To characterize these rates, it is also convenient to look at two different cases, according to the value of $\mu^j$, the separation between the mixture components.

\smallskip \noindent {\bf{Small $\mu^j$.}} First, we consider the case when each $||\mu^j||/\sigma^j$ is
less than a fixed constant $\sqrt{\ln \frac{9}{2 \pi}}$, including the case
when $||\mu^j||$ can be much less than $1$. In this case, the Gaussians are not even separated in terms of probability mass; 
in fact, as $||\mu^j||/\sigma^j$ decreases, the overlap in probability mass between the Gaussians tends to $1$. 
However, we show that \imeans\ can still do
something interesting, in terms of recovering the subspace containing the means of the distributions. 
Theorem~\ref{thm:smallmu} summarizes the convergence rate in this case.

\begin{theorem}[Small $\mu^j$]\label{thm:smallmu}
Let $||\mu^j||/\sigma^j < \sqrt{\ln \frac{9}{2 \pi}}$, for $j=1,2$. Then, there exist fixed constants $a_1$ and $a_2$, such that:
\begin{eqnarray*} 
\cos^2(\theta_t)(1 + a_1(M/V) \sin^2(\theta_t)) \leq \cos^2(\theta_{t+1}) \leq \cos^2 (\theta_t) (1 + a_2 (M/V) \sin^2(\theta_t)) 
\end{eqnarray*}
\end{theorem}

For a mixture of two identical Gaussians with equal mixing weights, we can conclude:

\begin{corollary}\label{cor1:smallmu}
For a mixture of two identical spherical Gaussians with equal mixing weights, standard deviation $1$, 
if $\bmu = ||\mu^1|| = ||\mu^2|| < \sqrt{\ln \frac{9}{2 \pi}}$, then,
\[ \cos^2(\theta_t) (1 + a_1' \bmu^2 \sin^2(\theta_t)) \leq \cos^2(\theta_{t+1}) \leq \cos^2(\theta_t) (1 + a_2'\bmu^2 \sin^2 (\theta_t)) \]
\end{corollary}

The proof follows by a combination of Lemma~\ref{lem:expr1}, and Lemma~\ref{lem:smalltau}. 
From Corollary~\ref{cor1:smallmu}, we observe that $\cos^2(\theta_t)$ grows by a factor of 
$(1 + \Theta(\bmu^2))$ 
in each iteration, except when $\theta_t$ is very close to $0$. This means that when \imeans\ is far from the actual solution, it approaches the solution at a consistently high rate. The convergence rate only grows slower, once $k$-means is very close to the actual solution.

\medskip \noindent {\bf{Large $\mu^j$.}} In this case, there exists a $j$ such that $||\mu^j||/\sigma^j \geq \sqrt{\ln \frac{9}{2 \pi}}$. 
In this regime, the Gaussians have small overlap in probability mass, yet, the distance between two samples from the same distribution is much greater than the separation between the distributions. Our guarantees for this case are summarized by Theorem \ref{thm:largemu}.

We see from Theorem \ref{thm:largemu} that there are two regimes of behavior of the convergence rate, depending on the value of $\max_j |\tau_t^j|/\sigma^j$. These regimes have a natural interpretation. The first regime corresponds to the case when $\theta_t$ is large enough, such that when projected onto $u_t$, at most a constant fraction of samples from the two distributions can be classified with high confidence. The second regime corresponds to the case when $\theta_t$ is close enough to $0$ such that when projected along $u_t$, most of the samples from the distributions can be classified with high confidence. As expected, in the second regime, the convergence rate is much slower than in the first regime.

\begin{theorem}[Large $\mu^j$]\label{thm:largemu}
Suppose there exists $j$ such that $||\mu^j||/\sigma^j \geq \sqrt{\ln \frac{9}{2 \pi}}$. 
If $|\tau_t^j|/\sigma^j < \sqrt{ \ln \frac{9}{2 \pi}}$, for all $j$, then, there exist fixed constants $a_3$, $a_4$, $a_5$ and $a_6$ such that:
\begin{eqnarray*}
\cos^2(\theta_t) \left(1 + \frac{a_3 (M/V)^2 \sin^2(\theta_t) }{a_4 + (M/V)^2 \cos^2(\theta_t)}\right) \leq \cos^2(\theta_{t+1}) \leq  \cos^2(\theta_t) \left(1 + \frac{a_5 ((M/V) + (M/V)^2)\sin^2(\theta_t)}{a_6 + (M/V)^2\cos^2(\theta_t)}\right) 
\end{eqnarray*}
On the other hand, if there exists $j$ such that $|\tau_t^j|/\sigma^j \geq \sqrt{\ln \frac{9}{2 \pi}}$, then, there exist fixed constants $a_7$ and $a_8$ such that:
\begin{eqnarray*}
\cos^2(\theta_t) (1 + \frac{ a_7 \rho_{\min}^2\mu_{\min}^2}{a_8 V^2 + \rho_{\min}^2 \mu_{\min}^2}\tan^2(\theta_t)) \leq \cos^2(\theta_{t+1}) \leq \cos^2(\theta_t)(1 + \tan^2(\theta_t)) 
\end{eqnarray*}
\end{theorem}

For two identical Gaussians with standard deviation $1$, we can conclude:

\begin{corollary}
For a mixture of two identical Gaussians with equal mixing weights, and standard deviation $1$, if $\bmu = ||\mu^1|| = ||\mu^2|| > \sqrt{\ln \frac{9}{2 \pi}}$, and if $|\tau_t^1|=|\tau_t^2| \leq \sqrt{\ln \frac{9}{2 \pi}}$, then, there exist fixed constants $a'_3, a'_4, a'_5, a'_6$ such that:
\begin{eqnarray*}
\cos^2(\theta_t) \left(1 + \frac{a'_3 \bmu^4 \sin^2(\theta_t) }{a'_4 + \bmu^4 \cos^2(\theta_t)}\right) \leq \cos^2(\theta_{t+1}) \leq  \cos^2(\theta_t) \left(1 + \frac{a'_5 \bmu^4\sin^2(\theta_t)}{a_6 + \bmu^4\cos^2(\theta_t)}\right) 
\end{eqnarray*}
On the other hand, if $|\tau_t^1|= |\tau_t^2| \geq \sqrt{\ln \frac{9}{2 \pi}}$, then, there exists a fixed constant $a'_7$ such that:
\begin{eqnarray*}
\cos^2(\theta_t) (1 +  a'_7 \tan^2(\theta_t)) \leq \cos^2(\theta_{t+1}) \leq \cos^2(\theta_t)(1 + \tan^2(\theta_t)) 
\end{eqnarray*}
\end{corollary}

In this case as well, we observe the same phenomenon: the convergence rate
is high when we are far away from the solution, and slow when we are close.
Using Theorems~\ref{thm:smallmu} and~\ref{thm:largemu}, we can characterize
the convergence times of \imeans; for the sake of simplicity, we present
the convergence time bounds for a mixture of two
spherical Gaussians with equal mixing weights and standard deviation $1$.
We recall that in this case \imeans~is exactly $2$-means.

\begin{corollary}[Convergence Time]\label{cor:convtime}
If $\theta_0$ is the initial angle between $\mu^1$ and $u_0$, then,
$\cos^2(\theta_N) \geq 1 - \epsilon$ after 
$N = C_0 \cdot \left(\frac{\ln(\frac{1}{\cos^2(\theta_0)})}{ \ln(1 + \bmu^2)} +  \frac{1}{\ln(1 + \epsilon)} \right)$
 iterations, where $C_0$ is a fixed constant.
\end{corollary}

\medskip\noindent{\textbf{Effect of Initialization.}} As apparent from
Corollary~\ref{cor:convtime}, the effect of initialization is only to ensure a lower bound on the value of $\cos(\theta_0)$. We illustrate below, two natural ways by which one can select $u_0$, and their effect on the convergence rate. For the sake of simplicity, we state these bounds for the case in which we have two identical Gaussians with equal mixing weights and standard deviation $1$. 

\begin{itemize}

\item First, one can choose $u_0$ uniformly at random from the surface of a
unit sphere in $\bbR^d$; in this case, $\cos^2(\theta_0) = \Theta(\frac{1}{d})$, 
with constant probability, and as a result, the convergence
time to reach $\cos^{-1}(1/\sqrt{2})$ is 
$O(\frac{\ln d}{\ln(1 +\bmu^2)})$.

\item A second way to choose $u_0$ is to set it to be a random sample from
the mixture; in this case, 
$\cos^2(\theta_0) = \Theta( \frac{(1 +\bmu)^2}{d})$ 
with constant probability, and the time to reach
$\cos^{-1}(1/\sqrt{2})$ is $O(\frac{\ln d}{\ln(1 +\bmu^2)})$. 
\end{itemize}
%

%
\iffalse

\begin{corollary} [Small $\mu$]
Let $\mu < \sqrt{\ln \frac{9}{2 \pi}}$ and suppose the initial vector
$\hu_0$ is picked uniformly at random from the surface of the unit sphere
in $\bbR^d$. Then, with constant probability, $\cos^2(\theta_T)\ge 1-\epsilon$ 
after $T=O(\frac{\log{d}}{\mu^2}+\frac{1}{\mu^2\epsilon})$ rounds.
\label{cor:smallmu}
\end{corollary}

\begin{corollary}[Small $\mu$]
Let $\mu < \sqrt{\ln \frac{9}{2 \pi}}$ and suppose the initial vector
$u_0$ is chosen to be a random sample from the mixture. Then, with constant probability, $\cos^2(\theta_T)\ge 1-\epsilon$ 
after $T=O(\frac{\log{d}}{\mu^2}+\frac{1}{\mu\epsilon})$ rounds.
\label{cor:smallmu2}
\end{corollary}
\fi

\iffalse

\begin{corollary}
Let $\mu < \sqrt{\ln \frac{9}{2 \pi}}$, and suppose that we start
$k$-means with a vector $\hu_0$. 
If $\theta_f \geq \cos^{-1}(1/2)$, then, with constant probability, $\theta_t \leq \theta_f$, after $\Theta( \frac{\ln (d \cos^2(\theta_f))}{\mu^2})$ rounds. Otherwise, $\theta_t \leq \theta_f$ after $\Theta( \frac{\ln d}{\mu^2} + \frac{1}{\mu^2 \sin^2 (\theta_f)})$ rounds.
\label{cor:timesmallmu}
\end{corollary}

\begin{corollary} [Small $\mu$]
Let $\mu < \sqrt{\ln \frac{9}{2 \pi}}$. Then, if
$k$-means is initialized with a unit random vector $\hu_0$, then 
$T=O(\frac{\log d}{\log(1+\mu^2)} + \frac{1}{\log(1+\epsilon\mu^2)})$ rounds of
the algorithm guarantee that $\cos^2(\theta_T) \ge 1-\epsilon$.  
\label{cor:smallmu}
\end{corollary}
\fi

\section{Finite Samples}
\label{sec:k2finsamples}
In this section, we analyze Algorithm \imeans, when we are required to
estimate the statistics at each round with a finite number of samples. We characterize the number of samples needed to ensure that \imeans\ makes progress in each round, and we also characterize the rate of progress when the required number of samples are available. 

The main result of this section is the following lemma, which characterizes
$\theta_{t+1}$, the angle between $\mu^1$ and the hyperplane
separator in \imeans, given $\theta_t$. Notice that now $\theta_t$
is a random variable, which depends on the samples drawn in
rounds $1, \ldots, t-1$, and given $\theta_t$, $\theta_{t+1}$ is a random
variable, whose value depends on samples in round $t$.  Also we use
$u_{t+1}$ as the center of partition $C_t$ in iteration $t+1$, and
$\bbE[u_{t+1}]$ is the expected center. Note that all the expectations in
round $t$ are conditioned on $\theta_t$. In addition, we use $S_{t+1}$ to
denote the quantity $\bbE[X \cdot 1_{X \in C_{t+1}}]$, where $1_{X \in
C_{t+1}}$ is the indicator function for the event $X \in C_{t+1}$, and the
expectation is taken over the entire mixture. Note that, $S_{t+1} =
\bbE[u_{t+1}] \Pr[X \in C_{t+1}] = Z_{t+1} \bbE[u_{t+1}]$. We use
$\hat{S}_{t+1}$ to denote the empirical value of $S_{t+1}$.

\begin{lemma}\label{lem:keysample}
If we use $n$ samples in iteration $t$, then, given $\theta_t$, with probability $1 -
2\delta$, 
\begin{eqnarray*}
\cos^2(\theta_{t+1}) \geq \cos^2(\theta_t) & \left(1 + \tan^2(\theta_t) \frac{
2\cos(\theta_t)\xi_t\m_t + \m_t^2}{ \xi_t^2 + 2\cos(\theta_t)\xi_t\m_t +
\m_t^2 + \Delta_2 } \right) - \left( \frac{\Delta_2 \cos^2(\theta_t) + 2 \Delta_1 (\m_t + \xi_t \cos(\theta_t))}{ \m_t^2 + \xi_t^2
+ 2 \xi_t \m_t \cos(\theta_t) + \Delta_2 } \right)
\end{eqnarray*}
where,
\begin{eqnarray*}
\Delta_1 &=& \frac{8 \log(4n/\delta)(\sigma_{\max} + \max_j ||\mu^j||)}{\sqrt{n}} \\
\Delta_2 &=& \frac{128 \log^2(8n/\delta)(\sigma_{\max}^2 d + \sum_j
||\mu^j||^2)}{n} + \frac{8
\log(n/\delta)}{\sqrt{n}}(\sigma_{\max}||S_{t+1}|| +
\max_j|\dot{S_{t+1}}{\mu^j}|) \enspace 
\end{eqnarray*}
\end{lemma}

The main idea behind the proof of Lemma~\ref{lem:keysample} is that we can
write $\cos^2(\theta_{t+1}) = \frac{\dot{\hat{S}_{t+1}}{\mu^1}^2}{||\mu^1||^2
||\hat{S}_{t+1}||^2}$. Next, we can use Lemma~\ref{lem:expr1}, and the
definition of $S_{t+1}$ to get an expression for
$\frac{\dot{S_{t+1}}{\mu^1}^2}{||S_{t+1}||^2 ||\mu^1||^2}$, and
Lemmas~\ref{lem:projconc} and~\ref{lem:normconc} to bound $\dot{\hat{S}_{t+1} -
S_{t+1}}{\mu^1}$, and $||\hat{S}_{t+1}||^2 - ||S_{t+1}||^2$. Plugging in all these values gives us a proof of Lemma~\ref{lem:keysample}. We also assume for the rest of the section that the number of samples $n$ is at most some polynomial in $d$, such that $\log(n) = \Theta(\log(d))$. 

The two main lemmas used in the proof of Lemma~\ref{lem:keysample} are Lemmas~\ref{lem:projconc} and~\ref{lem:normconc}. To state them, we need to define some notation. At time $t$, we use the notation 

\begin{lemma}\label{lem:projconc}
For any $t$, and for any vector $v$ with norm $||v||$, with probability at least $1 - \delta$,
\[ | \dot{\hS_{t+1} - S_{t+1}}{v} | \leq \frac{8 \log(4n/\delta)(\sigma_{\max} ||v|| + \max_j |\dot{\mu^j}{v}|)}{\sqrt{n}} \]
\end{lemma}
\begin{lemma}\label{lem:normconc}
For any $t$, with probability at least $1 - \delta$,
\[ ||\hS_{t+1}||^2 \leq ||S_{t+1}||^2 + \frac{128\log^2(8n/\delta)(\sigma_{\max}^2 d + \sum_j (\mu^j)^2)}{n} + \frac{16 \log(8n/\delta)}{\sqrt{n}} ( \sigma_{\max} ||S_{t+1}|| + \max_j |\dot{S_{t+1}}{\mu^j}|) \]
\end{lemma}

The proofs of Lemmas~\ref{lem:projconc} and ~\ref{lem:normconc} are in the
Appendix. Applying Lemma~\ref{lem:keysample}, we can characterize the
number of samples required such that \imeans\ makes progress in each round
for different values of $||\mu^j||$. Again, it is convenient to look at two
separate cases, based on $||\mu^j||$.

\begin{theorem}[Small $\mu^j$]\label{thm:finsmallmu}
Let $||\mu^j||/\sigma^j < \sqrt{\ln \frac{9}{2 \pi}}$, for all $j$. If the number of samples drawn in round
$t$ is at least
$a_9 \sigma_{\max}^2\log^2(d/\delta) \left( \frac{d}{M V\sin^4(\theta_t)}
+ \frac{1}{M^2 \sin^4(\theta_t)\cos^2(\theta_t)} \right)$, for some fixed constant $a_{9}$, then, with probability at least $1-\delta$,
$\cos^2({\theta}_{t+1}) \ge \cos^2(\theta_t)(1+ a_{10}(M/V)\sin^2(\theta_t))$,
where $a_{10}$ is some fixed constant.
\end{theorem}

In particular, for the case of two identical Gaussians with equal mixing
weights and standard deviation $1$, our results implies the following.

\begin{corollary} \label{thm:finsmallmuunif}
Let $\bmu = ||\mu^1|| = ||\mu^2|| < \sqrt{\ln \frac{9}{2 \pi}}$. If the number of samples drawn in
round $t$ is at least 
$a_{9} \log^2(d/\delta) \left( \frac{d}{\bmu^2\sin^4(\theta_t)} + \frac{1}{\bmu^4 \cos^2(\theta_t) \sin^4(\theta_t)} \right)$, 
for some fixed constant $a_{9}$, then, with probability at least $1-\delta$,
$\cos^2({\theta}_{t+1}) \ge \cos^2(\theta_t)(1+ a_{10}\bmu^2\sin^2(\theta_t))$,
where $a_{10}$ is some fixed constant.
\end{corollary}

In particular, when we initialize $u_0$ with a vector picked uniformly at
random from a $d$-dimensional sphere, $\cos^2(\theta_0) \geq \frac{1}{d}$,
with constant probability, and thus the number of samples required for
success in the first round is 
$\tilde{\Theta}(\frac{d}{\bmu^4})$.
This bound matches with the lower bounds for learning mixtures of Gaussians in one dimension~\cite{Lbook}, as well as with conjectured lower bounds in experimental work~\cite{SSR06}. The following corollary summarizes the total number of samples required to learn the mixture with some fixed precision, for two identical spherical Gaussians with variance $1$ and equal mixing weights.
 
\begin{corollary}\label{cor:smallmu1s}
Let $\bmu = ||\mu^1|| = ||\mu^2|| \leq \sqrt{\ln \frac{9}{2 \pi}}$. Suppose $u_0$ is chosen uniformly at random, 
and the number of rounds is 
$N \geq C_0 \cdot (\frac{\ln d}{ \ln(1 + \bmu^2)} + \frac{1}{\ln(1 + \epsilon)})$, 
where $C_0$ is the fixed constant in Corollary \ref{cor:convtime}. If the number of samples $|\samp|$ is at
least:
$ \frac{N \cdot a_{9} d\log^2(d)}{\bmu^4 \epsilon^2}$, 
then, with constant probability, after $N$ rounds, $\cos^2(\theta_N) \geq 1 - \epsilon$.
\end{corollary}

One can show a very similar corollary when $u_0$ is initialized as a random
sample from the mixture. We note that the total number of samples is a
factor of 
$N \approxeq \frac{\ln d}{\bmu^2}$ 
times greater than the bound in Theorem~\ref{thm:finsmallmu}. This is due to the fact that we use a fresh set of samples in every round, in order to simplify our analysis. In practice, successive iterations of $k$-means or EM is run on the same data-set.

\begin{theorem}[Large $\mu^j$]\label{thm:finlargemu} 
Suppose that there exists some $j$ such that $||\mu^j||/\sigma^j \geq \sqrt{\ln\frac{9}{2 \pi}}$, and suppose that the number of samples drawn in round
$t$ is at least 
\[ a_{11} \log^2(d/\delta) \left(
\frac{d\sigma_{\max}^2}{ \rho_{\min}^2 \mu_{\min}^2\sin^4(\theta_t)}
+ \frac{ \sigma_{\max}^2 + \max_j ||\mu^j||^2}{M^2 \cos^2(\theta_t)
\sin^4(\theta_t)}  + \frac{\sigma_{\max}^2 \max_j ||\mu^j||^2 + \max_j
||\mu^j||^4 }{\rho_{\min}^4 \mu_{\min}^4 \sin^4(\theta_t)} \right) \]
for some constant $a_{11}$. If $|\tau_t^j| \leq \sqrt{\ln \frac{9}{2 \pi}}$,
for all $j$, then, with probability at least $1 - \delta$,
$\cos^2({\theta}_{t+1}) \ge \cos^2(\theta_t)(1+ a_{12} \min(1, M^2 + MV)\sin^2(\theta_t))$; otherwise, with probability at least $1 - \delta$,
$\cos^2({\theta}_{t+1}) \ge \cos^2(\theta_t)(1+ a_{13} \frac{\rho_{\min}^2
\mu_{\min}^2 \tan^2(\theta_t)}{ V^2 + \rho_{\min}^2 \mu_{\min}^2} )$, where $a_{12}$ and $a_{13}$ are fixed constants.
\end{theorem}

For a mixture of two identical Gaussians with equal mixing weights and
standard deviation $1$, our result implies:

\begin{corollary}\label{cor:largemuunif}
Suppose that $\bmu = ||\mu^1||=||\mu^2||  \geq \sqrt{\ln \frac{9}{2 \pi}}$, and suppose that the number of samples in round $t$ is at least:
$a_{11}\log^2(d/\delta) \left( \frac{d}{\bmu^2\sin^4(\theta_t)} + \frac{1}{\bmu^2 \cos^2(\theta_t) \sin^4(\theta_t)}\right)$,
for some constant $a_{11}$. If $|\tau_t^j| \leq \sqrt{\ln \frac{9}{2 \pi}}$, then, with probability at least $1 - \delta$, $\cos^2({\theta}_{t+1}) \ge \cos^2(\theta_t)(1+ a_{12} \sin^2(\theta_t))$; otherwise, with probability $1 - \delta$,
$\cos^2({\theta}_{t+1}) \ge \cos^2(\theta_t)(1+ a_{13}\tan^2(\theta_t))$, where $a_{12}$ and $a_{13}$ are fixed constants.
\end{corollary}

Again, if we pick $u_0$ uniformly at random, we require about
$\tilde{\Omega}(\frac{d}{\bmu^2})$ 
samples for the first round to succeed. 
When 
$\bmu > 1$, 
this bound is worse than
$\frac{d}{\bmu^4}$, 
but matches with the upper bounds of~\cite{BCFZ07}. The following corollary
shows the number of samples required in total for \imeans~to
converge. 

\begin{corollary}\label{cor:largemu1s}
Let $\bmu \geq \sqrt{\ln \frac{9}{2 \pi}}$. Suppose $u_0$ is chosen uniformly at random and the number of rounds is 
$N \geq C_0 \cdot (\ln d + \frac{1}{\ln(1 + \epsilon)})$, where $C_0$ is
the constant in Corollary \ref{cor:convtime}. If  $|\samp|$ is at least
$\frac{2N C_0 d \log^2(d) }{\bmu^2 \epsilon^2}$, 
then, with constant probability, after $N$ rounds, $\cos^2(\theta_N) \geq 1 - \epsilon$.
\end{corollary}

\section{Lower Bounds}
\label{sec:lowerbounds}
In this section, we prove a lower bound on the sample complexity of
learning mixtures of Gaussians, using Fano's Inequality~\cite{Y97,CT90},
stated in Theorem~\ref{thm:fano}. Our main theorem in this section can be summarized as follows.

\begin{theorem}
Suppose we are given samples from the mixture $D(\mmu) = \frac{1}{2} \calN(\mmu,
I_d) + \frac{1}{2}\calN(-\mmu, I_d)$, for some $\mmu$, and let $\hat{\mmu}$ be the estimate
of $\mmu$ computed from $n$ samples. If $n < \frac{C d}{||\mmu||^2}$
for some constant $C$, and $||\mmu|| > 1$, then, there exists $\mmu$ such
that $\bbE_{D(\mmu)} ||\mmu - \hat{\mmu}|| \geq C' ||\mmu||$, where $C'$ is a constant.
\label{thm:lowermain}
\end{theorem}

The main tools in the proof of Theorem~\ref{thm:lowermain} are the following lemmas, and a generalized version of Fano's Inequality~\cite{CT90, Y97}.

\begin{lemma}
Let $\mmu_1, \mmu_2 \in \bbR^d$, and let $D_1$ and $D_2$ be the following mixture distributions: $D_1 = \frac{1}{2}
\calN(\mmu_1, I_d) + \frac{1}{2} \calN(-\mmu_1, I_d)$, and $D_2 =
\frac{1}{2} \calN(\mmu_2, I_d) + \frac{1}{2} \calN(-\mmu_2,
I_d)$. Then, 
\[ \kl(D_1, D_2) \leq  \frac{1}{\sqrt{2 \pi}}\cdot \left( ||\mmu_2||^2 -
||\mmu_1||^2 + \frac{3\sqrt{2 \pi}}{2} \ln 2 +
2||\mmu_1|| ( e^{-||\mmu_1||^2/2} + \sqrt{2 \pi} ||\mmu_1|| \Phi(0,
||\mmu_1||)) \right) \]
\label{lem:kllower}
\end{lemma}

\begin{lemma}
There exists a set of vectors $V = \{ v_1, \ldots, v_K\}$ in $\bbR^d$ with the following
properties: (1) For each $i$ and $j$, $d(v_i, v_j) \geq \frac{1}{5}, d(v_i, -v_j)
\geq \frac{1}{5}$. (2) $K = e^{d/10}$. (3) For all $i$, $||v_i|| \leq \sqrt{\frac{7}{5}}  $.
\label{lem:spherepacking}
\end{lemma}

\begin{theorem}[Fano's Inequality]
Consider a class of densities $F$, which contains $r$ densities $f_1,
\ldots, f_r$, corresponding to parameter values $\theta_1, \ldots,
\theta_r$. Let $d(\cdot)$ be any metric on
$\theta$, and let $\hat{\theta}$ be an estimate of $\theta$ from $n$
samples from a density $f$ in $F$. If, for all $i$ and $j$, $d(\theta_i, \theta_j) \geq \alpha$, and
$\kl(f_i, f_j) \leq \beta$, then,
$\max_{j} \bbE_j d(\hat{\theta}, \theta_j) \geq \frac{\alpha}{2} ( 1 -
\frac{n \beta + \log 2}{\log(r - 1)})$, where $\bbE_j$ denotes the
expectation with respect to distribution $j$.
\label{thm:fano}
\end{theorem}

\begin{proof}(Of Theorem~\ref{thm:lowermain})
We apply Fano's Inequality. Our class of densities $F$ is the
class of all mixtures of the form $\frac{1}{2} \calN(\mmu', I_d) +
\frac{1}{2} \calN(-\mmu', I_d)$. We set the parameter $\theta = \mmu'$, and
$d(\mmu_1, \mmu_2) = ||\mmu_1 - \mmu_2||$. 
We construct a subclass $\calF = \{f_1, \ldots, f_r \}$ of $F$ as follows. We
set each $f_i = \frac{1}{2} \calN(||\mmu|| v_i, I_d) + \frac{1}{2}
\calN(-||\mmu||v_i, I_d)$, for each vector $v_i$ in $V$ in
Lemma~\ref{lem:spherepacking}. Notice that now $r = e^{d/10}$. Moreover,
for each pair $i$ and $j$, from Lemma~\ref{lem:kllower} and
Lemma~\ref{lem:spherepacking}, $\kl(f_i, f_j)
\leq C_1 ||\mmu||^2 + C_2$, for constants $C_1$ and $C_2$. Finally, from
Lemma~\ref{lem:spherepacking}, for each pair $i$ and $j$, $d(\mmu_i,
\mmu_j) \geq \frac{||\mmu||}{5}$. The Theorem now follows by an application
of Fano's Inequality~\ref{thm:fano}.
\end{proof}

\section{More General $k$-means}\label{sec:genkinf}

In this section, we show that when we apply $2$-means on an input generated
by a mixture of $k$ spherical Gaussians, the normal to the hyperplane which
partitions the two clusters in the $2$-means algorithm, converges to a
vector in the subspace $\Mu$ containing the means of mixture components.
We assume that our input is generated by a mixture of $k$ spherical Gaussians, with means $\mu^j$, variances $(\sigma^j)^2$, $j = 1, \ldots, k$, and mixing weights $\rho^1, \ldots, \rho^k$.
The mixture is centered at the origin such that $\sum \rho^j \mu^j = 0$.
We use $\Mu$ to denote the subspace containing the means $\mu^1, \ldots,
\mu^k$. We use Algorithm~\imeans\ on this input, and our goal is to show that it
still converges to a vector in $\Mu$.

\noindent{\textbf{Notation.}} In the sequel, given a vector $x$ and a subspace $W$, we define the angle
between $x$ and $W$ as the angle between $x$ and the projection of $x$ onto
$W$. We examine the angle $\theta_{t}$, between
$u_t$ and $\Mu$, and our goal is to show that the cosine of this angle
grows as $t$ increases. Our main result of this section is
Lemma~\ref{lem:kexpr1}, which exactly defines the behavior of \imeans\ on a mixture of $k$ spherical Gaussians. 
Recall that at time $t$, we use $\hu_t$ to partition the input data, and the projection of $\hu_t$ along $\Mu$ is $\cos(\theta_t)$ by definition. 
Let $b^1_t$ be a unit vector lying in the subspace $\Mu$ such that: $\hu_t = \cos(\theta_t) b^1_t + \sin(\theta_t) v_t $,
where $v_t$ lies in the orthogonal complement of $\Mu$, and has norm $1$. We define a second vector $\hup_t$ as follows: $\hup_t = \sin(\theta_t) b^1_t - \cos(\theta_t) v_t $. We observe that $\dot{\hu_t}{\hup_t} = 0$, $||\hup_t|| = 1$, and the projection of $\hup_t$
on $\Mu$ is $\sin(\theta_t) b^1_t$.We now extend the set $\{b^1_t\}$ to complete an orthonormal basis $\calB = \{b^1_t, \ldots,
b^{k-1}_t\}$ of $\Mu$. We also observe that $\{b^2_t, \ldots, b^{k-1}_t, \hu_t, \hup_t \}$ is an orthonormal basis of the subspace spanned by any basis of $\Mu$, along
with $v_t$, and can be extended to a basis of $\bbR^d$.

For $j = 1, \ldots, k$, we define $\tau^j_t$ as follows: $\tau_j^t = \dot{\mu^j}{\hu_t} = \cos(\theta_t) \dot{\mu^j}{b^1_t}$.
Finally we (re)-define the quantity $\xi_t$, and define $m_t^l$, for $l=1, \ldots, k-1$ as
\[ \xi_t = \sum_j \rho^j \sigma^j  \frac{e^{-(\tau_t^j)^2/2(\sigma^j)^2}}{\sqrt{2 \pi}}, ~~~~~\m_t^l = \sum_j \rho^j \Phi(-\frac{\tau_t^j}{\sigma^j},\infty) \dot{\mu^j}{b^l_t} \]
Our main lemma is stated below. The proof is in the Appendix.

\begin{lemma}\label{lem:kexpr1}
At any iteration $t$ of Algorithm \imeans,
\[ 
\cos^2(\theta_{t+1})  = 
\cos^2(\theta_t) \left(1 + \tan^2(\theta_t)\frac{ 2 \cos(\theta_t) \xi_t \m_t^1 + \sum_l (\m_t^l)^2 }{ \xi_t^2 + 2 \cos(\theta_t) \xi_t \m_t^1 + \sum_l (\m_t^l)^2 } \right)
\]
\end{lemma}

\bibliography{kmeans} 
\bibliographystyle{alpha}
\section*{Appendix}

\subsection{Proof of Lemma~\ref{lem:expr1}}
In this section, we prove Lemma~\ref{lem:expr1}. First, we need some additional notation. 

\medskip \noindent \textbf{Notation.} We define, for $j=1,2$:
\begin{eqnarray*}
w^j_{t+1} & = & \Pr[x \sim D_j | x \in C_{t+1}] \\
u^j_{t+1} & = & \E[ x | x \sim D_j, x \in C_{t+1}]
\end{eqnarray*}
We observe that $u_{t+1}$ now can be written as:
\[ u_{t+1} =w^1_{t+1} u^1_{t+1} +w^2_{t+1} u^2_{t+1} \]
Moreover, we define $Z_{t+1} = \Pr[x \in C_{t+1}]$.

\medskip \noindent {\textbf{Proof of Lemma~\ref{lem:expr1}.}} We start by providing exact expressions for $w^1_{t+1}$ and $w^2_{t+1}$
with respect to the partition computed in the previous round $t$. These are used to
compute the projections of $u_{t+1}$ along the vectors $\hu_t$ and
$\mu_1-\dot{\mu_1}{\hu_t}\hu_t$, which finally leads to a proof of Lemma~\ref{lem:expr1}.
\begin{lemma}\label{lem:wts12}
In round $t$, for $j=1,2$,
$w^j_{t+1} = \frac{ \rho^j \Phi(-\frac{\tau_t^j}{\sigma^j}, \infty)}{Z_{t+1}}$.
\end{lemma}
\begin{proof}
We can write:
\[ w^j_{t+1} = \frac{ \Pr[x \in C_{t+1} | x \sim D_j] \Pr[x \sim D_j]}{
\Pr[x \in C_{t+1}]} \]
We note that $\Pr[x \sim D_j] = \rho^j$, and $\Pr[x \in C_{t+1}] = Z_{t+1}$.

As $D_j$ is a spherical Gaussian, for any $x$ generated from $D_j$, and for any vector $y$ orthogonal to $u_t$,
$\dot{y}{x}$ is distributed independently from $\dot{\hu_t}{x}$. Moreover, we
observe that $\dot{\hu_t}{x}$ is distributed as a Gaussian with mean
$\dot{\mu^j}{\hu_t} = \tau_t^j$ and standard deviation $\sigma^j$. Therefore,
\[ \Pr[x \in C_{t+1} | x \sim D_j] = \Pr_{x \sim D_j}[\dot{\hu_t}{x} > 0] = \Pr[N(\tau_t^j, \sigma^j) \geq 0] =
\Phi(-\frac{\tau_t^j}{\sigma^j}, \infty) \]
from which the lemma follows. 
\end{proof}

\begin{lemma}\label{lem:expr2}
For any $t$, $\dot{u_{t+1}}{\hu_t}  =  \frac{\xi_t + \m_t\cos(\theta_t)}{Z_{t+1}}$.
\end{lemma}
\begin{proof}
Consider a sample $x$ drawn from $D_j$. Then, $\dot{x}{\hu_t}$ is
distributed as a Gaussian with mean $\dot{\mu^j}{\hu_t} = \tau_t^j$ and
standard deviation $\sigma^j$. 
We recall that $\Pr[x \in C_{t+1}] = Z_{t+1}$. Therefore,  $\dot{u^{j}_{t+1}}{\hu_t}$ is equal to:  
\[ \frac{\bbE[x, x\in C_{t+1} | x \sim D_j]}{\Pr[x \in C_{t+1}| x \sim D_j]} = \frac{1}{\Pr[N(\tau_t^j, \sigma^j) > 0]} \cdot \int_{y=0}^{\infty}\frac{y e^{-(y - \tau_t^j)^2/2 (\sigma^j)^2}}{\sigma^j \sqrt{2 \pi}} dy \]
which is, again, equal to:
\begin{eqnarray*}
&& \frac{1}{ \Phi(-\frac{\tau_t^j}{\sigma^j}, \infty)} \left( \tau_t^j \int_{y=0}^{\infty} \frac{e^{-(y - \tau_t^j)^2/2 (\sigma^j)^2}}{\sigma^j\sqrt{2 \pi}} dy + \int_{y=0}^{\infty} \frac{(y - \tau_t^j)e^{-(y - \tau_t^j)^2/2 (\sigma^j)^2}}{\sigma^j\sqrt{2 \pi}} dy \right)\\
& = &\frac{1}{ \Phi(-\frac{\tau_t^j}{\sigma^j}, \infty)} \left( \tau_t^j \Phi(-\frac{\tau_t^j}{\sigma^j}, \infty) + \int_{y=0}^{\infty} \frac{(y - \tau_t^j)e^{-(y - \tau_t^j)^2/2 (\sigma^j)^2}}{\sigma^j\sqrt{2 \pi}} dy \right)
\end{eqnarray*}
We can compute the integral in the equation above as follows.
\begin{eqnarray*}
\int_{y = 0}^{\infty} (y-\tau_t^j)e^{-(y - \tau_t^j)^2/2(\sigma^j)^2} dy = (\sigma^j)^2 \int_{z = (\tau_t^j)^2/2(\sigma^j)^2}^{\infty} e^{-z}dz =  (\sigma^j)^2 e^{-(\tau_t^j)^2/2 (\sigma^j)^2}
\end{eqnarray*}

We can now compute $\dot{u_{t+1}}{\hu_t}$ as follows.

\begin{eqnarray*}
\dot{u_{t+1}}{\hu_t} = w^1_{t+1} \dot{u^1_{t+1}}{\hu_t} +  w^2_{t+1}
\dot{u^2_{t+1}}{\hu_t} =  \frac{1}{Z_{t+1}} \cdot \sum_j \left( \rho^j \tau_t^j \Phi(-\frac{\tau_t^j}{\sigma^j}, \infty) + \rho^j (\sigma^j)^2 \frac{e^{-(\tau_t^j)^2/2 (\sigma^j)^2}}{\sigma^j\sqrt{2 \pi}} \right)
\end{eqnarray*}

The lemma follows by recalling $\tau_t^j = \dot{\mu^j}{b}\cos(\theta_t)$ and plugging in the values of $m_t$ and $\xi_t$.

\end{proof}

\begin{lemma}\label{lem:expr3}
Let $\hv_t$ be a unit vector along $\mu_1 - \dot{\mu_1}{\hu_t}\hu_t$. Then, $\dot{u_{t+1}}{\hv_t} = \frac{\m_t\sin(\theta_t)}{Z_{t+1}}$.
In addition, for any vector $z$ orthogonal to $\hu_t$ and $\hv_t$,
$\dot{u_{t+1}}{z} = 0$.
\end{lemma}

\begin{proof}
We observe that for a sample $x$ drawn from distribution $D_1$
(respectively, $D_2$) and any unit vector $v_1$, orthogonal to $\hu_t$, 
$\dot{x}{v_1}$ is distributed as a Gaussian with mean $\dot{\mu^1}{v_1}$
($\dot{\mu^2}{v_1}$, respectively) and standard deviation $\sigma^1$ (resp. $\sigma^2$).
Therefore, 
the projection of $u_{t+1}$ on $\hv_t$ can be written as:
\begin{eqnarray*}
\dot{u_{t+1}}{\hv_t} & = & \sum_j w^j_{t+1} \dot{\mu^j}{\hv_t} = \frac{1}{Z_{t+1}} \sum_j \rho^j \Phi(-\frac{\tau_t^j}{\sigma^j}, \infty) \dot{\mu^j}{\hv_t} 
\end{eqnarray*} 

from which the first part of the lemma follows.

The second part of the lemma follows from the observation that for any
vector $z$ orthogonal to $\hu_t$ and $\hv_t$, $\dot{\mu^j}{z} = 0$, for $j = 1,2$.
\end{proof}

\begin{lemma}\label{lem:expr4}
For any $t$,
\begin{eqnarray*}
\dot{u_{t+1}}{\mu^1} & = & \frac{||\mu^1|| (\xi_t\cos(\theta_t) + \m_t)}{Z_{t+1}}\\
||u_{t+1}||^2 & = & \frac{\xi_t^2  + \m_t^2 + 2\xi_t\m_t\cos(\theta_t)}{(Z_{t+1})^2}
\end{eqnarray*}
\end{lemma}
\begin{proof}
As we have an infinite number of samples, $\theta_{t+1}$ lies on the
same plane as $\theta_t$. Therefore, we can write
$\dot{u_{t+1}}{\mu^1} =  \dot{u_{t+1}}{\hu_t} \dot{\mu^1}{\hu_t} + \dot{u_{t+1}}{\hv_t} \dot{\mu^1}{ \hv_t}$. 
Moreover, we can write $||u_{t+1}||^2 = \dot{u_{t+1}}{\hu_t}^2 + \dot{u_{t+1}}{\hv_t}^2$.
Thus, the first two equation follow by using Lemma \ref{lem:expr2} and
\ref{lem:expr3}, and recalling that $\dot{\mu^1}{\hu_t} = \tau_t^1 = ||\mu^1||\cos(\theta_t)$ and 
$\dot{\mu^1}{\hv_t}=||\mu^1||\sin(\theta_t)$.
\end{proof}

We are now ready to complete the proof of Lemma \ref{lem:expr1}. 
\begin{proof}(Of Lemma \ref{lem:expr1})
By definition of $\theta_{t+1}$, $\cos^2(\theta_{t+1}) = \frac{\dot{u_{t+1}}{\mu^1}^2}{||u_{t+1}||^2 ||\mu^1||^2}$. Therefore,
\begin{eqnarray*}
||\mu^1||^2 \cos^2  (\theta_{t+1}) & = &  \frac{\dot{u_{t+1}}{\mu^1}^2}{||u_{t+1}||^2}\\
& = &  (\tau_t^1)^2 \left(1+\frac{\dot{u_{t+1}}{\mu^1}^2-||\mu^1||^2\cos^2(\theta_t)||u_{t+1}||^2}{||\mu^1||^2\cos^2(\theta_t)||u_{t+1}||^2}\right)\\
& = & (\tau_t^1)^2 \left(1 + \frac{||\mu^1||^2\sin^2(\theta_t) (\m_t^2+2\xi_t\m_t\cos(\theta_t))}{||\mu^1||^2\cos^2(\theta_t)||u_{t+1}||^2}\right)\\
& = & ||\mu^1||^2\cos^2(\theta_t) \left(1 + \tan^2(\theta_t)\frac{\m_t^2+2\xi_t\m_t\cos(\theta_t)}{||u_{t+1}||^2}\right)
\end{eqnarray*}
where we used Lemma~\ref{lem:expr4} and the observation that $\cos(\theta_t) = \frac{\tau_t^1}{||\mu^1||}$.
The Lemma follows by replacing $||u_{t+1}||^2$ using the expression in Lemma~\ref{lem:expr4}.
\end{proof}

The next Lemma helps us to derive Theorem~\ref{thm:smallmu} from Lemma~\ref{lem:expr1}. 
It shows how to approximate $\Phi(-\tau, \tau)$ when $\tau$ is small.

\begin{lemma}\label{lem:smalltau}
Let $\tau \leq \sqrt{ \ln \frac{9}{2 \pi}}$. Then, $\frac{5}{3 \sqrt{2 \pi}} \tau \leq \Phi(-\tau, \tau) \leq \frac{2}{\sqrt{2
\pi}} \tau$. In addition, $\frac{2e^{-\tau^2/2}}{\sqrt{2 \pi}} \geq
\frac{2}{3}$.
\end{lemma}

\subsection{Proofs of Sample Requirement Bounds}

For the rest of the section, we prove Lemmas~\ref{lem:projconc} and~\ref{lem:normconc}, which lead to a proof of Lemma~\ref{lem:keysample}. First, we need to define some notation.

\medskip \noindent {\textbf{Notation.}} At time $t$, we use the notation $S_{t+1}$ to denote the quantity $\bbE[X
\cdot 1_{X \in C_{t+1}}]$, where $1_{X \in C_{t+1}}$ is the indicator
function for the event $X \in C_{t+1}$, and the expectation is taken
over the entire mixture. 

In the sequel, we also use the notation $\hS_{t+1}$ to denote the empirical
value of $S_{t+1}$. Our goal is to bound the concentration of certain
functions of $\hS_{t+1}$ around their expected values, when we are given
only $n$ samples from the mixture.
Recall that we define $\theta_{t+1}$ as the angle between $\mu^1$ and the hyperplane separator in \imeans, given $\theta_t$. Notice that now $\theta_t$
is a random variable, which depends on the samples drawn in
rounds $1, \ldots, t-1$, and given $\theta_t$, $\theta_{t+1}$ is a random
variable, whose value depends on samples in round $t$.  Also we use
$u_{t+1}$ as the center of partition $C_t$ in iteration $t+1$, and
$\bbE[u_{t+1}]$ is the expected center. Note that all the expectations in
round $t$ are conditioned on $\theta_t$.

\medskip \noindent {\textbf{Proofs.}} We are now ready to prove Lemmas~\ref{lem:projconc} and~\ref{lem:normconc}.

\begin{proof}(Of Lemma~\ref{lem:projconc})
Let $X_1, \ldots, X_n$ be the $n$ iid samples from the mixture; for each
$i$, we can write the projection of $X_i$ along $v$ as follows:
\[ \dot{X_i}{v} = Y_i + Z_i \]
where $Z_i \sim N(0,\sigma^j)$, if $X_i$ is generated from distribution $D^j$, and $Y_i =\dot{\mu^j}{v}$, if $X_i$ is generated by
$D^j$. Therefore, we can write:
\[ \dot{\hS_{t+1}}{v} = \frac{1}{n} \left( \sum_i Y_i \cdot 1_{X_i \in C_{t+1}} + \sum_i Z_i \cdot 1_{X_i \in C_{t+1}} \right) \]
To determine the concentration of $\dot{\hS_{t+1}}{v}$ around its expected
value, we address the two terms separately.

The first term is a sum of $n$ independently distributed random variables,
such that changing one variable changes the sum by at most $\max_j \frac{2
|\dot{\mu^j}{v}|}{n}$; therefore, to calculate its concentration, one can apply
Hoeffding's Inequality. It follows that with probability at most $\frac{\delta}{2}$,
\[ 
|\frac{1}{n} \sum_i Y_i \cdot 1_{X_i \in C_{t+1}} - \bbE[\frac{1}{n}
\sum_i Y_i  \cdot 1_{X_i \in C_{t+1}}] |  > \max_j \frac{4|\dot{\mu^j}{v}|
\sqrt{\log(4n/\delta)}}{\sqrt{n}} \]

We note that, in the second term, each $Z_i$ is a Gaussian with mean $0$
and variance $\sigma^j$, scaled by $||v||$. For some $0 \leq \delta' \leq 1$, let
$E_i(\delta')$ denote the event
\[ -\sigma_{\max}||v|| \sqrt{2 \log(1/\delta')} \leq Z_i \cdot 1_{X_i \in C_{t+1}} \leq
\sigma_{\max} ||v|| \sqrt{2 \log (1/\delta')} \]
As $Z_i \sim N(0,\sigma^j)$, if $X_i$ is generated from distribution $D_j$, and $1_{X_i \in C_{t+1}}$ takes values $0$ and $1$,
for any $i$, for $\delta'$ small enough,$\Pr[E_i(\delta')] \geq 1 - \delta'$.

We use $\delta' = \frac{\delta}{4n}$, and condition on the fact that all the
events $\{ E_i(\delta'), i=1, \ldots, n\}$ happen; using an Union bound
over the events $\bar{E_i(\delta')}$, the probability that this holds is at
least $1 - \frac{\delta}{4}$. We also observe that, as the Gaussians $Z_i$
are independently distributed, conditioned on the union of the events
$E_i$, the Gaussians $Z_i$ are still independent. Therefore, conditioned on
the event $\cup_i E_i(\delta')$, $\frac{1}{n} \sum_i Z_i \cdot 1_{X_i \in
C_{t+1}}$ is the sum of $n$ independent random variables, such that
changing one variable changes the sum by at most $\frac{2\sigma_{\max}||v||\sqrt{2
\log(1/\delta')}}{n}$. We can now apply Hoeffding's bound to conclude that with probability at least $1 -
\frac{\delta}{2}$,
\[| \frac{1}{n}  \sum_i Z_i \cdot 1_{X_i \in C_{t+1}} - \bbE[\frac{1}{n}\sum_i Z_i \cdot 1_{X_i \in C_{t+1}}] | \leq \frac{4\sigma_{\max}||v||\sqrt{2 \log(1/\delta')} \sqrt{2 \log(1/\delta)}}{\sqrt{n}} \leq \frac{8 \sigma_{\max} ||v||
\log(4n/\delta)}{\sqrt{n}} \]

The lemma now follows by applying an union bound.
\end{proof}

\begin{proof} (Of Lemma~\ref{lem:normconc})
We can write:
\[ ||\hS_{t+1}||^2 \leq ||S_{t+1}||^2 + ||\hS_{t+1} - S_{t+1}||^2 +
2|\dot{\hS_{t+1} - S_{t+1}}{S_{t+1}}| \]

If $v_1, \ldots, v_d$ is any orthonormal basis of $\bbR^d$, then, we can
bound the second term as follows. With probability at least $1 -
\frac{\delta}{2}$,
\begin{eqnarray*}
 ||\hS_{t+1} - S_{t+1}  ||^2 & = & \sum_{i=1}^{d} (\dot{\hS_{t+1} - S_{t+1}}{v_i})^2 \leq \ \frac{128 \log^2(8n/\delta)}{n} (\sum_i \sigma_{\max}^2 ||v_i||^2 + \sum_{i,j} \dot{\mu^j}{v_i}^2) \\
&  \leq & \frac{128 \log^2(8n/\delta)}{n}(\sigma_{\max}^2 d + \sum_j (\mu^j)^2) 
\end{eqnarray*}

\iffalse
%
\begin{eqnarray*}
||\hS_{t+1} - S_{t+1}||^2   =  \sum_{i=1}^{d} (\dot{\hS_{t+1} -
S_{t+1}}{v_i})^2 \\
%
%
\ \leq \ \frac{128 \log^2(8n/\delta)}{n} (\sum_i ||v_i||^2 + \sum_i
\dot{\mu}{v_i}^2) \\
\ \leq \ \frac{128 \log^2(8n/\delta)}{n}(d + \mu^2) 
\end{eqnarray*}
\fi
The second step follows by the application of Lemma \ref{lem:projconc}, and
the fact that for any $a$ and $b$, $(a + b)^2 \leq
2(a^2 + b^2)$.

Using Lemma \ref{lem:projconc}, with probability at least $1 -
\frac{\delta}{2}$,
\[ \dot{\hS_{t+1} - S_{t+1}}{S_{t+1}} \leq \frac{8
\log(8n/\delta)}{\sqrt{n}}(\sigma_{\max} ||S_{t+1}|| + \max_j |\dot{S_{t+1}}{\mu^j}|) \]
The lemma follows by a union bound over these two above events.
\end{proof}

\subsection{Proofs of Lower Bounds}

\begin{proof}(Of Lemma~\ref{lem:kllower})
Let $P$ be the plane containing the origin $O$ and the vectors $\mmu_1$ and $\mmu_2$. If
$v$ is a vector orthogonal to $P$, then, the projection of $D_1$ along $v$ is a Gaussian $\calN(0,1)$, which is distributed independently of the
projection of $D_1$ along $P$ (and same is the case for $D_2$).Therefore,
to compute the KL-Divergence of $D_1$ and $D_2$, it is sufficient to
compute the KL-Divergence of the projections of $D_1$ and $D_2$ along the
plane $P$.

Let $x$ be a vector in $P$. Then, 
\begin{eqnarray*}
\kl(D_1, D_2) & = & \frac{1}{\sqrt{2 \pi}}\int_{x \in P} (\frac{1}{2}e^{-||x - \mmu_1||^2/2} +
\frac{1}{2} e^{-||x + \mmu_1||^2/2}) \ln \left(\frac{ \frac{1}{2}e^{-||x -
\mmu_1||^2/2} +
\frac{1}{2} e^{-||x + \mmu_1||^2/2}}{ \frac{1}{2}e^{-||x - \mmu_2||^2/2} +
\frac{1}{2} e^{-||x + \mmu_2||^2/2}} \right) dx \\
& = &  \frac{1}{\sqrt{2 \pi}}\int_{x \in P} (\frac{1}{2}e^{-||x - \mmu_1||^2/2} +
\frac{1}{2} e^{-||x + \mmu_1||^2/2}) \ln \left( \frac{ e^{-||x +
\mmu_1||^2/2}
\cdot (1 + e^{2 \dot{x}{\mmu_1}})}{ e^{-||x + \mmu_2||^2/2} \cdot (1 +
e^{2 \dot{x}{\mmu_2}})} \right) dx \\
& = &  \frac{1}{\sqrt{2 \pi}}\int_{x \in P} (\frac{1}{2}e^{-||x - \mmu_1||^2/2} +
\frac{1}{2} e^{-||x + \mmu_1||^2/2})\left( ( ||x + \mmu_2||^2 - ||x +
\mmu_1||^2) + \ln \frac{1 + e^{2 \dot{x}{\mmu_1}}}{1 + e^{2
\dot{x}{\mmu_2}}} \right)dx 
\end{eqnarray*}

We observe that for any $x$, $||x + \mmu_2||^2 - ||x + \mmu_1||^2 =
||\mmu_2||^2 - ||\mmu_1||^2 + 2\dot{x}{\mmu_2 - \mmu_1}$. As the expected
value of $D_1$ is $0$, we can write that:
\begin{equation}
\int_{x \in P} (\frac{1}{2}e^{-||x - \mmu_1||^2/2} +
\frac{1}{2} e^{-||x + \mmu_1||^2/2}) \dot{x}{\mmu_2 - \mmu_1} = \bbE_{x \sim
D_1} \dot{x}{\mmu_1 - \mmu_2} = 0 
\label{eqn:lowerterm1}
\end{equation}

We now focus on the case where $||\mmu_1|| >> 1$. We observe that for any $\mmu_2$ and any $x$, $1 + e^{2 \dot{x}{\mmu_2}} > 1$. Therefore, combining the previous two equations,
\[ \kl(D_1, D_2) \leq  \frac{1}{\sqrt{2 \pi}}\left( ||\mmu_2||^2 - ||\mmu_1||^2 + \int_{x \in P}
(\frac{1}{2}e^{-||x - \mmu_1||^2/2} +
\frac{1}{2} e^{-||x + \mmu_1||^2/2}) \ln (1 + e^{2 \dot{x}{ \mmu_1}}) dx
\right)\]
Again, since the projection of $D_1$ perpendicular to $\mmu_1$ is
distributed independently of the projection of $D_1$ along $\mmu_1$, the
above integral can be taken over a one-dimensional $x$ which varies along
the vector $\mmu_1$. For the rest of the proof, we abuse notation, and use
$\mmu_1$ to denote both the vector $\mmu_1$ and the scalar $||\mmu_1||$.
We can write:
\begin{eqnarray*}
&& \int_{x=-\infty}^{\infty}  (\frac{1}{2}e^{-(x - \mmu_1)^2/2} +
\frac{1}{2} e^{-(x + \mmu_1)^2/2}) \ln(1 + e^{2 \mmu_1 x}) dx \\
& \leq & \sqrt{2 \pi} \ln 2 + \int_{x = 0}^{\infty}  (\frac{1}{2}e^{-(x - \mmu_1)^2/2} +
\frac{1}{2} e^{-(x + \mmu_1)^2/2}) \ln(1 + e^{2 \mmu_1 x}) dx \\
& \leq & \sqrt{2 \pi} \ln 2 +  \int_{x = 0}^{\infty}  (\frac{1}{2}e^{-(x - \mmu_1)^2/2} +
\frac{1}{2} e^{-(x + \mmu_1)^2/2}) (\ln 2 + 2 x \mmu_1) dx \\
& \leq & \frac{3 \sqrt{2 \pi}}{2} \ln 2 +  2 \mmu_1 \int_{x = 0}^{\infty}
(\frac{1}{2}e^{-(x - \mmu_1)^2/2} +
\frac{1}{2} e^{-(x + \mmu_1)^2/2}) x dx
\end{eqnarray*}
The first part follows because for $x < 0$, $\ln (1 + e^{2 x \mmu_1})
\leq \ln 2$. The second part follows because for $x > 0$, $\ln (1 + e^{2 x
\mmu_1}) \leq \ln (2 e^{2 x\mmu_1})$. The third part follows from the
symmetry of $D_1$ around the origin.

Now, for any $a$, we can write:
\begin{eqnarray*}
 \frac{1}{\sqrt{2 \pi}}\int_{x = 0}^{\infty} xe^{-(x+a)^2/2} dx =
 \frac{1}{\sqrt{2 \pi}} \cdot e^{-a^2/2} - a \Phi(a, \infty)
\end{eqnarray*}
Plugging this in, we can show that,
\[ \kl(D_1, D_2) \leq  \frac{1}{\sqrt{2 \pi}}\left( ||\mmu_2||^2 -
||\mmu_1||^2 + \frac{3\sqrt{2 \pi}}{2} \ln 2 +
2||\mmu_1|| ( e^{-||\mmu_1||^2/2} + \sqrt{2 \pi} ||\mmu_1|| \Phi(0,
||\mmu_1||)) \right) \]
from which the lemma follows.
\end{proof}
\begin{proof}(Of Lemma \ref{lem:spherepacking})
For each $i$, let each $v_i$ be drawn independently from the distribution
$\frac{1}{\sqrt{d}}\calN(0, I_d)$. For each $i, j$, let $P_{ij}
=\frac{d}{2} \cdot d(v_i,
v_j)$ and $N_{ij} = \frac{d}{2} \cdot d(v_i, -v_j)$. Then, for each $i$ and
$j$, $P_{ij}$ and $N_{ij}$ are distributed according to the Chi-squared
distribution with parameter $d$. From Lemma~\ref{lem:chisq}, it follows
that: $\Pr[ P_{ij} < \frac{d}{10}] \leq e^{-3d/10}$. A similar lemma can
also be shown to hold for the random variables $N_{ij}$.
Applying the Union Bound, the probability that this holds for $P_{ij}$ and
$N_{ij}$ for all pairs $(i,j), i \in V, j \in V$ is at most $2K^2
e^{-3d/10}$. This probability is at most $\frac{1}{2}$ when $K = e^{d/10}$.

In addition, we observe that for each vector $v_i$, $d \cdot ||v_i||^2$ is
also distributed as a Chi-squared distribution with parameter $d$.
From Lemma~\ref{lem:chisq}, for each $i$, $\Pr[ ||v_i||^2 > 7/5] \leq
e^{-2d/15}$. The second part of the lemma now follows by an Union Bound
over all $K$ vectors in the set $V$.
\end{proof}

\begin{lemma}
Let $X$ be a random variable, drawn from the Chi-squared distribution with
parameter $d$. Then, 
\[ \Pr[X < \frac{d}{10}] \leq e^{-3d/10} \]
Moreover, 
\[ \Pr[X > \frac{7d}{5}] \leq e^{-2d/15} \]
\label{lem:chisq}
\end{lemma}

\begin{proof}
Let $Y$ be the random variable defined as follows: $Y = d - X$. Then,
\[ \Pr[ X < \frac{d}{10}] = \Pr[ Y > \frac{9d}{10}] = \Pr[ e^{tY} >
e^{9dt/10}] \leq \frac{\bbE[e^{tY}]}{e^{9dt/10}} \]
where the last step uses a Markov's Inequality.
We observe that $\bbE[e^{tY}] = e^{td}\bbE[e^{-tX}] = e^{td}(1 -
2t)^{d/2}$, for $t < \frac{1}{2}$. The first part of the lemma follows from the observation
that $(1 - 2t)^{d/2} \leq e^{-td}$, and by plugging in $t =
\frac{1}{3}$.  

For the second part, we again observe that
\[ \Pr[ X > \frac{7d}{5}] \leq (1 - 2t)^{-d/2} e^{-7dt/5} \leq e^{-2dt/5}
\]
The lemma now follows by plugging in $t = \frac{1}{3}$.
\end{proof}

\subsection{More General $k$-means : Results and Proofs}\label{sec:genkinfappendix}

In this section, we show that when we apply $2$-means on an input generated
by a mixture of $k$ spherical Gaussians, the normal to the hyperplane which
partitions the two clusters in the $2$-means algorithm, converges to a
vector in the subspace $\Mu$ containing the means of mixture components.
This subspace is interesting because, in this subspace, the distance between the
means is as high as in the original space; however, if the number of
clusters is small compared to the dimension, the distance between two
samples from the same cluster is much smaller. In fact, several algorithms
for learning mixture models~\cite{VW02, AM05, CR08} attempt to isolate this
subspace first, and then use some simple clustering methods in this subspace.

\subsubsection{The Setting}

We assume that our input is generated by a mixture of $k$ spherical Gaussians, with means
$\mu^j$, variances $(\sigma^j)^2$, $j = 1, \ldots, k$, and mixing weights $\rho^1, \ldots, \rho^k$.
The mixture is centered at the origin such that $\sum \rho^j \mu^j = 0$.
We use $\Mu$ to denote the subspace containing the means $\mu^1, \ldots,
\mu^k$.

We use Algorithm~\imeans\ on this input, and our goal is to show that it
still converges to a vector in $\Mu$. 

%
\iffalse
We also provide rates of convergence
for Algorithm~\imeans. To characterize this convergence rate, we need an assumption, which we call the {\em rank condition}. 

We use  $\mu$ to denote the following quantity:
\[ \mu = \sqrt{\sum_j \rho^j ||\mu^j||^2} \]
We note that, if we assigned weights $\rho^j$ to point mass $\mu^j$, then
$\mu^2$ is the Frobenius norm of the covariance matrix of these point
masses. Our assumption is thus a statement about the projection of this matrix along any direction, and can be interpreted as a condition on the dimension of the space spanned by the weighted mean vectors.  

\begin{assumption}(Rank Condition)
 Let $M$ be the matrix $\sum_j \rho^j \mu^j (\mu^j)^{\top}$. Then, the
 minimum non-zero eigenvalue of $M$ is at least $\lambda_{\min}^2 \mu^2$.
\end{assumption}

Moreover, we let $\lambda_{\max}^2 = 1 - (k-1) \lambda_{\min}^2$; this is an upper bound on the maximum eigenvalue of the matrix $M$.
\fi
%

In the sequel, given a vector $x$ and a subspace $W$, we define the angle
between $x$ and $W$ as the angle between $x$ and the projection of $x$ onto
$W$. As in Sections 2 and 3, we examine the angle $\theta_{t}$, between
$u_t$ and $\Mu$, and our goal is to show that the cosine of this angle
grows as $t$ increases. Our main result of this section is
Lemma~\ref{lem:kexpr1}, which, analogous to Lemma \ref{lem:expr1} in
Section \ref{sec:k2infsamples}, exactly defines the behavior of $2$-means on a mixture 
of $k$ spherical Gaussians.

Before we can prove the lemma, we need some additional notation.

\subsubsection{Notation}

Recall that at time $t$, we use $\hu_t$ to partition the input data, and the projection of $\hu_t$ along $\Mu$ is $\cos(\theta_t)$ by definition. 
Let $b^1_t$ be a unit vector lying in the subspace $\Mu$ such that:
\[ \hu_t = \cos(\theta_t) b^1_t + \sin(\theta_t) v_t \]
where $v_t$ lies in the orthogonal complement of $\Mu$, and has norm $1$. We define a second vector $\hup_t$ as follows:
\[ \hup_t = \sin(\theta_t) b^1_t - \cos(\theta_t) v_t \]
We observe that $\dot{\hu_t}{\hup_t} = 0$, $||\hup_t|| = 1$, and the projection of $\hup_t$
on $\Mu$ is $\sin(\theta_t) b^1_t$.

We now extend the set $\{b^1_t\}$ to complete an orthonormal basis $\calB = \{b^1_t, \ldots,
b^{k-1}_t\}$ of $\Mu$. We also observe that $\{b^2_t, \ldots, b^{k-1}_t, \hu_t, \hup_t \}$
is an orthonormal basis of the subspace spanned by any basis of $\Mu$, along
with $v_t$, and can be extended to a basis of $\bbR^d$.

For $j = 1, \ldots, k$, we define $\tau^j_t$ as follows:
\[ \tau_j^t = \dot{\mu^j}{\hu_t} = \cos(\theta_t) \dot{\mu^j}{b^1_t} \]

Finally we (re)-define the quantity $\xi_t$ as
\[ \xi_t = \sum_j \rho^j \sigma^j  \frac{e^{-(\tau_t^j)^2/2(\sigma^j)^2}}{\sqrt{2 \pi}} \]
and, for any $l = 1, \ldots, k-1$, we define:
\[ \m_t^l = \sum_j \rho^j \Phi(-\frac{\tau_t^j}{\sigma^j},\infty) \dot{\mu^j}{b^l_t} \]

\subsubsection{Proof of Lemma \ref{lem:kexpr1}}

The main idea behind the proof of Lemma \ref{lem:kexpr1} is to estimate the norm and the projection of $u_{t+1}$; we do this in three steps. First, we estimate the projection of $u_{t+1}$ along $\hu_t$; next, we estimate this projection on $\hup_t$, and finally, we estimate its projection along $b^2_t, \ldots, b^l_t$. Combining these projections, and observing that the projection of $u_{t+1}$ on any direction perpendicular to these is $0$, we can prove the lemma.

As before, we define
\[ Z_{t+1} = \Pr[ x \in C_{t+1}] \]

Now we make the following claim.
\begin{lemma}\label{lem:wts}
For any $t$ and any $j$,
\[ \Pr[x \sim D_j | x \in C_{t+1}] = \frac{\rho^j}{Z_{t+1}} \Phi(-\frac{\tau_t^j}{\sigma^j}, \infty) \]
\end{lemma}
\begin{proof}
Same proof of Lemma~\ref{lem:wts12}
\end{proof}

Next, we estimate the projection of $u_{t+1}$ along $\hu_t$.

\begin{lemma}\label{lem:proj1}
\[ \dot{u_{t+1}}{\hu_t} = \frac{ \xi_t + \cos(\theta_t) \m_t^1 }{Z_{t+1}} \]
\end{lemma}
\begin{proof}
Consider a sample $x$ drawn from distribution $D_j$. The projection of $x$
on $\hu_t$ is distributed as a Gaussian with mean $\tau_t^j$ and standard deviation $\sigma^j$. 
The probability that $x$ lies in $C_{t+1}$ is $\Pr[N(\tau_t^j,\sigma^j)>0] = \Phi(-\frac{\tau_t^j}{\sigma^j}, \infty)$. 
Given that $x$ lies in $C_{t+1}$, the projection of $x$ on $\hu_t$ is distributed as a truncated
Gaussian, with mean $\tau_t^j$ and standard deviation $\sigma^j$, which is truncated at $0$.
Therefore, 
\[
\bbE[\dot{x}{\hu_t} | x \in C_{t+1}, x \sim D_j] =
 \frac{1}{\Phi(-\frac{\tau_t^j}{\sigma^j}, \infty)}
\left( \int_{y=0}^{\infty} \frac{y e^{-(y - \tau_t^j)^2/2}}{\sigma^j\sqrt{2 \pi}} dy
\right) 
\]
which is again equal to 
\begin{eqnarray*}
&& \frac{1}{ \Phi(-\frac{\tau_t^j}{\sigma^j}, \infty)} \left( \tau_t^j \int_{y=0}^{\infty} \frac{e^{-(y - \tau_t^j)^2/2 (\sigma^j)^2}}{\sigma^j\sqrt{2 \pi}} dy + \int_{y=0}^{\infty} \frac{(y - \tau_t^j)e^{-(y - \tau_t^j)^2/2 (\sigma^j)^2}}{\sigma^j\sqrt{2 \pi}} dy \right)\\
& = &\frac{1}{ \Phi(-\frac{\tau_t^j}{\sigma^j}, \infty)} \left( \tau_t^j \Phi(-\frac{\tau_t^j}{\sigma^j}, \infty) + \int_{y=0}^{\infty} \frac{(y - \tau_t^j)e^{-(y - \tau_t^j)^2/2 (\sigma^j)^2}}{\sigma^j\sqrt{2 \pi}} dy \right)
\end{eqnarray*}
We can evaluate the integral in the equation above as follows.
\begin{eqnarray*}
\int_{y = 0}^{\infty} (y-\tau_t^j)e^{-(y - \tau_t^j)^2/2(\sigma^j)^2} dy = (\sigma^j)^2 \int_{z = (\tau_t^j)^2/2(\sigma^j)^2}^{\infty} e^{-z}dz =  (\sigma^j)^2 e^{-(\tau_t^j)^2/2 (\sigma^j)^2}
\end{eqnarray*}
Therefore we can conclude that 
\[ \bbE[\dot{x}{\hu_t} | x \in C_{t+1}, x \sim D_j] = \tau_t^j + \frac{1}{\Phi(-\frac{\tau_t^j}{\sigma^j}, \infty)} \cdot \sigma^j \frac{e^{-(\tau_t^j)^2/2 (\sigma^j)^2}}{\sqrt{2 \pi}}
\]
Now we can write
\begin{eqnarray*}
\dot{u_{t+1}}{\hu_t} & = & \sum_j\bbE[\dot{x}{\hu_t} | x \sim D_j, x \in C_{t+1}]\Pr[x \sim D_j | x \in C_{t+1}] \\
& = & \frac{1}{Z_{t+1}}\sum_j\rho^j \Phi(-\frac{\tau_t^j}{\sigma^j}, \infty) \bbE[\dot{x}{\hu_t} | x \sim D_j, x \in C_{t+1}]  \\
\end{eqnarray*}
where we used lemma~\ref{lem:wts}. The lemma follows by recalling that $\tau_t^j = \cos(\theta_t) \dot{\mu^j}{b^1_t}$.
\end{proof}

\begin{lemma}\label{lem:proj2}
For any $t$,
\[ \dot{u_{t+1}}{\hup_t} = \frac{\sin(\theta_t) \m_t^1}{Z_{t+1}} \]
\end{lemma}
\begin{proof}
Let $x$ be a sample drawn from distribution $D_j$. Since $\hup_t$ is
perpendicular to $\hu_t$, and $D_j$ is a spherical Gaussian, given that $x \in C_{t+1}$, that is, the
projection of $x$ on $\hu_t$ is greater than $0$, the
projection of $x$ on $\hup_t$ is still distributed as a Gaussian with mean
$\dot{\mu^j}{\hup_t}$ and standard deviation $\sigma^j$. That is,
\[ \bbE[ \dot{x}{\hup_t} | x \sim D_j, x \in C_{t+1}] = \dot{\mu^j}{\hup_t} \]
Also recall that, by definition of $\hup_t$, $\dot{\mu^j}{\hup_t} = \sin(\theta_t)\dot{\mu^j}{b^1_t}$.
To prove the lemma, we observe that $\dot{u_{t+1}}{\hup_t}$ is equal to
\[ \sum_j \bbE[\dot{x}{\hup_t} | x \sim D_j, x \in C_{t+1}] \Pr[x \sim D_j | x \in C_{t+1}] \]
The lemma follows by using lemma~\ref{lem:wts}.
\end{proof}

\begin{lemma}\label{lem:proj3}
For $l \geq 2$,
\[ \dot{u_{t+1}}{b^l_t} = \frac{\m_t^l}{Z_{t+1}} \]
\end{lemma}
\begin{proof}
Let $x$ be a sample drawn from distribution $D_j$. Since $b^l_t$ is
perpendicular to $\hu_t$, and $D_j$ is a spherical Gaussian, given that $x \in C_{t+1}$, that is, the
projection of $x$ on $\hu_t$ is greater than $0$, the
projection of $x$ on $b^l_t$ is still distributed as a Gaussian with mean
$\dot{\mu^j}{b^l_t}$ and standard deviation $\sigma^j$. That is,
\[ \bbE[ \dot{x}{b^l_t} | x \sim D_j, x \in C_{t+1}] = \dot{\mu^j}{b^l_t} \]
To prove the lemma, we observe that $\dot{b^l_t}{u_{t+1}}$ is equal to
\[ \sum_j \bbE[\dot{x}{b^l_t} | x \sim D_j, x \in C_{t+1}] \Pr[x \sim D_j | x \in C_{t+1}] \]
The lemma follows by using lemma~\ref{lem:wts}.
\end{proof}

Finally, we show a lemma which estimates the norm of the vector $u_{t+1}$.

\begin{lemma}\label{lem:proj4}
\begin{eqnarray*}
||u_{t+1}||^2 & = & \frac{1}{Z_{t+1}^2}( \xi_t^2 + 2 \xi_t \cos(\theta_t) \m_t^1 + \sum_{l=1}^{k} (\m_t^l)^2 )
\end{eqnarray*}
\end{lemma}

\begin{proof}
Combining Lemmas \ref{lem:proj1}, \ref{lem:proj2} and \ref{lem:proj3}, we can write:
\begin{eqnarray*}
||u_{t+1}||^2 & = & \dot{\hu_t}{u_{t+1}}^2 + \dot{\hup_t}{u_{t+1}}^2 + \sum_{l
\geq 2} \dot{b^l_t}{u_{t+1}}^2 \\
& = & \frac{1}{Z_{t+1}^2} \bigg( \xi_t^2 + 2 \xi_t \cos(\theta_t) \m_t^1 + \cos^2(\theta_t) (\m_t^1)^2  + \sin^2(\theta_t) (\m_t^1)^2 + \sum_{l=2}^{k} (\m_t^l)^2 \bigg) 
\end{eqnarray*}
The lemma follows by plugging in the fact that $\cos^2(\theta_t) + \sin^2(\theta_t) = 1$.
\end{proof}

Now we are ready to prove Lemma \ref{lem:kexpr1}.

\begin{proof}(Of Lemma \ref{lem:kexpr1})
Since $b_t^1, \ldots, b_t^k$ form a basis of $\Mu$, we can write:
\begin{equation}
\cos^2(\theta_{t+1}) = \frac{ \sum_{l=1}^{k} \dot{u_{t+1}}{b_t^l}^2}{||u_{t+1}||^2} 
\label{eqn:cosk}
\end{equation}
$||u_{t+1}||^2$ is estimated in Lemma \ref{lem:proj4}, and $\dot{u_{t+1}}{b_t^l}$ is estimated by Lemma \ref{lem:proj2}. Using these lemmas, as $b_t^1$ lies in the subspace spanned by the orthogonal vectors $\hu_t$ and $\hup_t$, we can write:
\begin{eqnarray*}
\dot{u_{t+1}}{b_t^1} & = & \dot{\hu_t}{u_{t+1}} \dot{\hu_t}{b_t^1} + \dot{\hup_t}{u_{t+1}} \dot{\hup_t}{b_t^1} \\
& = & \frac{ \cos(\theta_t) \xi_t + \m_t^1}{Z_{t+1}}  
\end{eqnarray*} 
Plugging this in to Equation \ref{eqn:cosk}, we get:
\[ \cos^2(\theta_{t+1}) = \frac{ \xi_t^2 \cos^2(\theta_t) + 2 \xi_t \cos(\theta_t) \m_t^1 + \sum_l (\m_t^l)^2 }{ \xi_t^2 + 2 \xi_t \cos(\theta_t) \m_t^1 + \sum_l (\m_t^l)^2 } \]
The lemma follows by rearranging the above equation, similar to the proof of Lemma \ref{lem:expr1}.
\end{proof}

\end{document}